\documentclass[12pt]{article}
\usepackage{amsmath,amsthm,comment}
\usepackage{amsfonts}
\usepackage{multirow}
\usepackage{booktabs}
\usepackage{graphicx}
\usepackage{epstopdf}
\usepackage{xcolor}
\usepackage[title]{appendix}
\usepackage[top=1in, bottom=1in, left=1in, right=1in]{geometry}
\usepackage{amssymb}
\usepackage{stmaryrd}
\usepackage{cleveref}
\usepackage{graphicx} 
\usepackage{caption}
\usepackage{mwe}
\usepackage{graphics,graphicx,epsfig,psfrag,subfigure,threeparttable}
\usepackage{float}
\usepackage{color}
\usepackage{authblk}

\usepackage{booktabs}

\usepackage{placeins}

\usepackage{adjustbox}

\usepackage{tikz}
\usetikzlibrary{matrix,decorations.pathreplacing, calc, positioning,fit, arrows.meta}

\usepackage{comment}
\usepackage{mathtools}
\usepackage{amsmath,amssymb,amsfonts, amsthm}

\usepackage[linesnumbered,ruled,vlined]{algorithm2e}
\SetKwInput{kwInit}{Initialize}

\DeclareMathOperator*{\argmin}{arg\,min}

\newcommand{\bA}{\mathbf{A}}

\newcommand{\be}{\mathbf{e}}

\newcommand{\by}{\mathbf{y}}

\newcommand{\bF}{\mathbf{F}}
\newcommand{\bM}{\mathbf{M}}

\newcommand{\bI}{\mathbf{I}}
\newcommand{\bG}{\mathbf{G}}

\newcommand{\bE}{\mathbf{E}}
\newcommand{\bB}{\mathbf{B}}

\newcommand{\bX}{\mathbf{X}}
\newcommand{\bY}{\mathbf{Y}}
\newcommand{\bU}{\mathbf{U}}
\newcommand{\bV}{\mathbf{V}}
\newcommand{\bW}{\mathbf{W}}
\newcommand{\bZ}{\mathbf{Z}}
\newcommand{\bR}{\mathbf{R}}
\newcommand{\bP}{\mathbf{P}}
\newcommand{\bQ}{\mathbf{Q}}
\newcommand{\bT}{\mathbf{T}}

\newcommand{\bS}{\mathbf{S}}
\newcommand{\bH}{\mathbf{H}}

\newcommand{\bone}{\mathbf{1}}

\newcommand{\bepsilon}{\boldsymbol{\epsilon}}

\newcommand{\bSigma}{\boldsymbol{\Sigma}}

\newcommand{\bzero}{\mathbf{0}}

\def\trans{^{\rm T}}
\usepackage{mathabx}
\def\wh{\widehat}
\def\wt{\widetilde}

\newcommand{\norm}[1]{\Vert#1\Vert}
\newcommand{\Norm}[1]{\left\Vert#1\right\Vert}

\newcommand{\Abs}[1]{\left\vert#1\right\vert}

\newtheorem{remark}{Remark}
\newtheorem{assumption}{Assumption}
\newtheorem{theorem}{Theorem}
\newtheorem{lemma}{Lemma}
\newtheorem{definition}{Definition}

\title{Transductive Matrix Completion with Calibration for Multi-Task Learning\thanks{An abridged version of this paper will appear in the 2023 IEEE International Conference on Acoustics, Speech, and Signal Processing (ICASSP 2023).\\
E-mail addresses: hengfang@fjnu.edu.cn, yasminzhang@ucla.edu, maoxj@sjtu.edu.cn, wangzl@xmu.edu.cn.}}
\author[1]{\normalsize Hengfang Wang}
\author[2]{Yasi Zhang}
\author[3]{Xiaojun Mao}
\author[4]{Zhonglei Wang}
\affil[1]{School of Mathematics and Statistics, Fujian Normal University, China}
\affil[2]{Department of Statistics, University of California, Los Angeles, USA}
\affil[3]{School of Mathematical Sciences, Ministry of Education Key Laboratory of Scientific and Engineering Computing, Shanghai Jiao Tong University, China}
\affil[4]{Wang Yanan Institute for Studies in Economics and School of Economics, Xiamen University, China}
\date{}

\begin{document}
\maketitle

\begin{abstract}
    Multi-task learning has attracted much attention due to growing multi-purpose research with multiple related data sources. Moreover, transduction with matrix completion is a useful method in multi-label learning. In this paper, we propose a transductive matrix completion algorithm that incorporates a calibration constraint for the features under the multi-task learning framework. The proposed algorithm recovers the incomplete feature matrix and target matrix simultaneously. Fortunately, the calibration information improves the completion results. In particular, we provide a statistical guarantee for the proposed algorithm, and the theoretical improvement induced by calibration information is also studied. Moreover, the proposed algorithm enjoys a sub-linear convergence rate. Several synthetic data experiments are conducted, which show the proposed algorithm out-performs other existing methods, especially when the target matrix is associated with the feature matrix in a nonlinear way.
  \end{abstract}
\section{Introduction}
With the advent of the big data era, massive amounts of information and data have been collected by modern high-tech devices, including web servers, environmental sensors, x-ray imaging machines, and so on. Based on multiple data sources related to learning purposes,  algorithms have been developed to achieve various learning goals. Multi-task learning (MTL) \cite{caruana1997multitask}, for example, implements a robust learner for multiple tasks incorporating multiple sources, and it can be widely used in practice, including web search, medical diagnosis, natural language processing, and computer version.

MTL has advantages for analyzing the multi-response structure. For $i = 1, \ldots, n$, let $\by_{i} \in \mathbb{R}^{m}$ be the  vector of interest of length $m$, where $n$ is the number of instances, and each element corresponds to a particular task.  Denote $\bY = (y_{ij})$ to be the $n\times m$ response matrix, consisting of the $n$ realizations of $\by_{i}$. If the responses for a specific task are categorical, it is essentially a multi-label learning problem. As for continuous responses, it turns out to be a multi-response regression problem if additional features are available. As the exponential family is flexible to handle different data types, regardless of categorical or continuous, for each task, we assume that $y_{ij}$ 
independently comes from an exponential family distribution with parameter $z_{\star, ij}$, which forms $\bZ_{\star} = (z_{\star, ij})$. The goal is to learn the underlying $\bZ_{\star}$ based on $\bY$. Moreover,  entries of $\bY$ may suffer from missingness which makes it  difficult  to learn $\bZ_{\star}$. To tackle this difficulty and estimate $\bZ_{\star}$ simultaneously, matrix completion (MC) \cite{recht2011simpler, keshavan2009matrix, koltchinskii2011nuclear,negahban2012restricted,cai2013max} algorithms have been developed under the low-rank and other regularity assumptions.   
 The recovery of the target matrix without any additional features is studied in \cite{alaya2019collective}.  Noisy MC is investigated in \cite{fan2019generalized,robin2019main} from the exponential family with  fully observed features.

In MTL problems,  features related to the responses may exist, but it is inevitable that such features  also suffer from missingness.
Goldberg et al. (2010) \cite{goldberg2010transduction} studied  such a problem under multi-label learning. Specifically, they studied single-task binary response matrix given feature matrix  where both matrices were subject to missingness. 
Xu et al. (2018) \cite{xu2018matrix} proposed the co-completion method by additionally penalizing the trace norm of the feature matrix. A recent work \cite{esmaeili2018transduction} considered this problem via a smoothed rank function approach.  Under the noiseless case,  \cite{chiang2018using} studied the error rate in the scenario with  corrupted side information.

Calibration \cite{deville1992calibration,sarndal2003model,fuller2011sampling} is widely used to  incorporate such prior knowledge in the learning procedure, and it improves the  estimation efficiency. In this vein, 
\cite{breidt2000local} studied the calibration with local polynomial regression  between responses and features. Later,  \cite{wu2001model} generalized this idea to the so-called model calibration under a parametric framework. Afterward,   \cite{kim2010calibration} generalized the idea to a functional type calibration equation. As far as we know, no calibration related work has been done  under the MC framework.

In this paper, we focus on MTL problems incorporating incomplete features, and we assume that certain prior knowledge about the features is also available. Our work can be embodied in the following toy example. For the well-known Netflix movie rating problem \cite{bennett2007netflix}, we aim to simultaneously complete  rating and like-dislike matrices incorporating an incomplete feature matrix, consisting of age, gender, and so on. However, all three matrices  suffer from missingness.  When additional information such as the summary statistics for age, gender, etc., can be obtained  from the census survey, we investigate the benefits of such additional information incorporated by the calibration method.  In a nutshell, we propose  a \textit{Transductive Matrix Completion with Calibration} (TMCC) algorithm, and the prior information about the  features is considered by an additional calibration constraint.
As far as we know, 
this is the first paper exploring multi-task learning under a matrix completion framework with calibration. 
     Methodologically, our method has two merits: (i) the target and feature matrices can be completed simultaneously; (ii) calibration information can be incorporated.
      Theoretically, we show the statistical guarantee of our method, and the benefit of calibration is also discussed. Besides, we have validated that the convergence  rate of the proposed algorithm is $O(1/k^{2})$.
    Numerically, synthetic data experiments show that our proposed method performs better than other existing methods, especially when the target matrix has a nonlinear transformation from the feature matrix.

\textit{Notations}. Denote $[n]$  as the set $\{1,\ldots, n\}$. Given an arbitrary matrix $\bS=(s_{ij}) \in \mathbb{R}^{n_{1} \times n_{2}}$, the Frobenius norm of $\bS$ is $\Norm{\bS}_{F} = ( \sum_{i=1}^{n_{1}}\sum_{j=1}^{n_{2}} s_{ij}^2  )^{1/2}$. Denote the singular values of $\bS$ as $\sigma_{1}(\bS),\ldots, \sigma_{n_{1}\wedge n_{2}}(\bS)$ in descending order.  The operator norm is  $\norm{\bS} = \sigma_{\max}(\bS)=\sigma_{1}(\bS)$ and the nuclear norm $\norm{\bS}_{\star} = \sum_{i=1}^{n_{1}\wedge n_{2} }\sigma_{i}(\bS)$. Besides, let $\sigma_{\min}(\bS) = \sigma_{n_{1}\wedge n_{2}}(\bS)$.

\section{Model and Algorithm}

Suppose there are $S$ response matrices, $\bY^{(1)}, \ldots, \bY^{(S)}$. For example, $\bY^{(1)}$ can be a like-dislike matrix, and $\bY^{(2)}$ can be a rating matrix, whose rows correspond to users, columns correspond to movies, and $S=2$. 
Denote the number of instances by $n$,  the number of tasks for $s$-th response matrix by $m_s$. For  $s \in [S]$, we have $\bY^{(s)} = (y_{ij}^{(s)}) \in \mathbb{R}^{n\times m_{s}}$, and  $y_{ij}^{(s)}$ can be either discrete or continuous  for $i\in[n]$.   We assume that within the same response matrix, all entries follow the same generic exponential family. For example, all entries of $\bY^{(1)}$ follow Bernoulli distributions with different success probabilities.  Let $\bR_{y}^{(s)} = (r_{y,ij}^{(s)}) \in \mathbb{R}^{n \times m_{s}}$ be the corresponding  indicator  matrix for $\bY^{(s)}$. In particular, if $y_{ij}^{(s)}$ is observed, then $r_{y,ij}^{(s)}=1$; otherwise,  $r_{y,ij}^{(s)}=0$. Furthermore,  assume that $\bY^{(s)}$ is generated from a low-rank matrix $\bZ_{\star}^{(s)} = (z_{\star, ij}^{(s)}) \in\mathbb{R}^{n\times m_s}$ by the exponential family via a base function $h^{(s)}$ and a  link function $g^{(s)}$,  namely, the density function $f^{(s)}(y_{ij}^{(s)}|z_{\star,ij}^{(s)})=h^{(s)}(y_{ij}^{(s)})\exp\{y_{ij}^{(s)}z_{\star,ij}^{(s)}-g^{(s)}(z_{\star,ij}^{(s)})\}$,  for  $s\in[S]$.  For instance, suppose $g^{(s)}(z_{\star,ij}^{(s)}) = \sigma^{2}(z_{\star,ij}^{(s)})^{2}/2$ and $h^{(s)}(y_{ij}^{(s)}) = (2\pi \sigma^{2})^{-1/2}\exp{ \{ -(y_{ij}^{(s)})^{2}/(2\sigma^{2})  \} }$, and the corresponding exponential family is the Gaussian distribution with mean $\sigma^{2}z_{\star,ij}^{(s)}$  and variance $\sigma^{2}$ with support $\mathbb{R}$.

Denote $\bX_{\star} = (x_{\star,ij})\in\mathbb{R}^{n\times d}$ as the true feature matrix consisting of $d$ feature, and it is assumed to be low-rank. 
Let $\bX =  (x_{ij})\in\mathbb{R}^{n\times d}$ be a noisy version of the true feature matrix, i.e., $\bX = \bX_{\star} + \bepsilon$, where $\bepsilon = (\epsilon_{ij}) \in \mathbb{R}^{n \times d}$ is a noise matrix with $\mathbb{E}(\mathbb{\bepsilon}) = \bzero$, and its entries are independent.
As the feature matrix is also incomplete, in a similar fashion, we denote $\bR_{x}=(r_{x,ij}) \in\mathbb{R}^{n\times d}$ as the corresponding indicator   matrix of $\bX$. That is, we only observe an incomplete matrix $\bR_{x} \circ \bX$, where $\circ$ denotes the Hadamard product. 
Let target matrix $\bZ_{\star} = [\bZ_{\star}^{(1)}, \ldots, \bZ_{\star}^{(S)}]$ be the collection of hidden parameter matrices. Our goal is to recover $\bZ_{\star}$. We believe that some hidden connection between $\bX_{\star}$ and $\bZ_{\star}$ may provide us benefits for MTL.

Our method can be illustrated in Fig~\ref{TIKZ}, where $\mbox{logit}(p) = \log\{p/(1-p)\}$, for $p \in (0,1)$.  In this example, we have three observed matrices whose entries are from Bernoulli, Gaussian, and Poisson distributions, correspondingly, and are subject to missingness. In addition, we have an incomplete feature matrix related to the target matrix and calibration information for the true feature matrix.  By the proposed TMCC algorithm, we can complete the true feature matrix and target matrix simultaneously.

\tikzset{%
  thick arrow/.style={
     -{Triangle[angle=90:1pt 1]},
     line width=0.3cm, 
     draw=blue!20 
  },
  arrow label/.style={
    text=black,
    font=\sf,
    align=center
  },
  set mark/.style={
    insert path={
      node [sloped, midway, arrow label, node contents=#1]
    }
  }
}

\begin{figure}[!t]
  \centering
\begin{tikzpicture}\tikzstyle{every node}=[scale=0.6]
 


\draw [thick arrow]
    (10, 0) arc (90:-80:3 and 3) 
  -- (10,-6);
  
  \node[rotate = -90] at  (10 + 3,-3) {TMCC (simultaneously recover)};

\draw [->, thick] (0,-1.75) -- (0,-3)  -- (12.5,-3) node [midway, above, sloped] {with calibration constraint $\bA\bX_{\star} = \bB$} ;

  \matrix  [nodes = draw, nodes={ minimum width=9mm,
  minimum height = 8mm}] (m)
  {
    \node{1.5}; & \node{?}; &  \node {-0.3};  &  \node{$\cdots$}; &  \node {-1.1};   \\
    \node{?}; &      \node{0.1}; &   \node {-0.5}; &  \node{$\cdots$};  &       \node {?}; \\
    \node{2.3}; &      \node{?}; &   \node {?}; &  \node{$\cdots$};  &       \node {?}; \\
    \node{$\vdots$}; &  \node{$\vdots$}; &  \node{$\vdots$}; &  \node{$\ddots$};  &      \node{$\vdots$}; \\
    \node{?}; &      \node{?}; &   \node {-2.2}; &  \node{$\cdots$};  &       \node {2.9}; \\
  };
  \draw [dashed, red, thick] (1.39,-2.5) -- (1.39,2.5);
    \draw (0, 1.75) node {Feature};
  \begin{scope}[xshift = 2.78 cm]
 \matrix  [nodes = draw, nodes={fill=yellow!20,minimum width=9mm,
 minimum height = 8mm}] 
  {
    \node{?}; & \node{0}; &  \node {1};  &  \node{$\cdots$}; &  \node {?};   \\
    \node{1}; &      \node{?}; &   \node {?}; &  \node{$\cdots$};  &       \node {0}; \\
    \node{?}; &      \node{?}; &   \node {0}; &  \node{$\cdots$};  &       \node {?}; \\
    \node{$\vdots$}; &  \node{$\vdots$}; &  \node{$\vdots$}; &  \node{$\ddots$};  &      \node{$\vdots$}; \\
    \node{?}; &      \node{1}; &   \node {?}; &  \node{$\cdots$};  &       \node {1}; \\
  };
    \draw (0,1.75) node {Bernoulli};
  \end{scope}
    \begin{scope}[xshift= 5.56cm]
 \matrix  [nodes = draw, nodes={fill=green!20,minimum width=9mm,
 minimum height = 8mm}] 
  {
    \node{-0.4}; & \node{0.7}; &  \node {-0.3};  &  \node{$\cdots$}; &  \node {?};   \\
    \node{?}; &      \node{0.1}; &   \node {?}; &  \node{$\cdots$};  &       \node {-0.5}; \\
    \node{?}; &      \node{?}; &   \node {0.1}; &  \node{$\cdots$};  &       \node {?}; \\
    \node{$\vdots$}; &  \node{$\vdots$}; &  \node{$\vdots$}; &  \node{$\ddots$};  &      \node{$\vdots$}; \\
    \node{0.1}; &      \node{?}; &   \node {?}; &  \node{$\cdots$};  &       \node {0.2}; \\
  };
    \draw (0,1.75) node {Gaussian};
    \draw (0,2.25) node[blue] {\Large Observed Matrix};
  \end{scope}
   \begin{scope}[xshift=  8.34 cm]
 \matrix  [nodes = draw, nodes={fill=pink!20,minimum width=9mm,
 minimum height = 8mm}] 
  {
    \node{1}; & \node{?}; &  \node {?};  &  \node{$\cdots$}; &  \node {2};   \\
    \node{?}; &      \node{?}; &   \node {1}; &  \node{$\cdots$};  &       \node {?}; \\
    \node{2}; &      \node{?}; &   \node {?}; &  \node{$\cdots$};  &       \node {2}; \\
    \node{$\vdots$}; &  \node{$\vdots$}; &  \node{$\vdots$}; &  \node{$\ddots$};  &      \node{$\vdots$}; \\
    \node{?}; &      \node{4}; &   \node {?}; &  \node{$\cdots$};  &       \node {3}; \\
  };
    \draw (0,1.75) node {Poisson};
  \end{scope}

  \begin{scope}[ yshift = -6cm]


  \matrix  [nodes = draw, nodes={minimum width=9mm,
  minimum height = 8mm}] at (0, 0)
  {
    \node{1.5}; & \node{1.3}; &  \node {-0.3};  &  \node{$\cdots$}; &  \node {-1};   \\
    \node{1.2}; &      \node{0}; &   \node {-0.5}; &  \node{$\cdots$};  &       \node {1.3}; \\
    \node{2.1}; &      \node{-0.3}; &   \node {-1.5}; &  \node{$\cdots$};  &       \node {-1.2}; \\
    \node{$\vdots$}; &  \node{$\vdots$}; &  \node{$\vdots$}; &  \node{$\ddots$};  &      \node{$\vdots$}; \\
    \node{3}; &      \node{-0.4}; &   \node {-2}; &  \node{$\cdots$};  &       \node {2.9}; \\
  };
      \draw [decorate,decoration={brace,amplitude=5pt}, xshift=-0.5pt, yshift=0pt]
(-1.39,-1.23) -- (-1.39, 1.23) node [black,midway,xshift=-0.6cm] 
{\footnotesize $n$};
    \draw [decorate,decoration={brace,amplitude=5pt,mirror}, xshift=0pt, yshift=-0.5pt]
(-1.39,-1.23) -- (1.39,-1.23) node [black,midway,yshift=-0.6cm] 
{\footnotesize $d$};
\draw [dashed, red, thick] (1.39,-2.5) -- (1.37,2.5);
    \draw (0,1.75) node {noiseless feature };
    \end{scope}
  \begin{scope}[xshift = 2.78 cm , yshift = - 6cm]
 \matrix  [nodes = draw, nodes={fill=yellow!20,minimum width=9mm,
 minimum height = 8mm}] 
  {
    \node{0.3}; & \node{0.2}; &  \node {0.8};  &  \node{$\cdots$}; &  \node {0.2};   \\
    \node{0.7}; &      \node{0.6}; &   \node {0.7}; &  \node{$\cdots$};  &       \node {0.1}; \\
    \node{0.2}; &      \node{0.7}; &   \node {0.1}; &  \node{$\cdots$};  &       \node {0.3}; \\
    \node{$\vdots$}; &  \node{$\vdots$}; &  \node{$\vdots$}; &  \node{$\ddots$};  &      \node{$\vdots$}; \\
    \node{0.1}; &      \node{0.9}; &   \node {0.3}; &  \node{$\cdots$};  &       \node {0.6}; \\
  };
    \draw [decorate,decoration={brace,amplitude=5pt,mirror}, xshift=0pt, yshift=-0.5pt]
    (-1.39,-1.23) -- (1.39,-1.23) node [black,midway,yshift=-0.6cm] 
{\footnotesize $m_{1}$};
    \draw (0,1.75) node [text width=3cm, align = center]{logit of probability (Bernoulli)};
  \end{scope}
    \begin{scope}[xshift=5.56cm , yshift = -6cm]
 \matrix  [nodes = draw, nodes={fill=green!20,minimum width=9mm,
 minimum height = 8mm}] 
    {
    \node{-0.2}; & \node{0.8}; &  \node {-0.3};  &  \node{$\cdots$}; &  \node {0.9};   \\
    \node{0.6}; &      \node{0.2}; &   \node {-1.1}; &  \node{$\cdots$};  &       \node {-0.5}; \\
    \node{0.4}; &      \node{0.5}; &   \node {0.1}; &  \node{$\cdots$};  &       \node {-1.3}; \\
    \node{$\vdots$}; &  \node{$\vdots$}; &  \node{$\vdots$}; &  \node{$\ddots$};  &      \node{$\vdots$}; \\
    \node{0.2}; &      \node{1.1}; &   \node {-0.7}; &  \node{$\cdots$};  &       \node {0.3}; \\
  };
    \draw [decorate,decoration={brace,amplitude=5pt,mirror}, xshift=0pt, yshift=-0.5pt]
    (-1.39,-1.23) -- (1.39,-1.23) node [black,midway,yshift=-0.6cm] 
{\footnotesize $m_{2}$};
    \draw (0,1.75) node {mean (Gaussian)};
        \draw (0,2.25) node[blue] {\Large Target Matrix};
  \end{scope}
   \begin{scope}[xshift=8.34cm  , yshift = -6cm]
 \matrix  [nodes = draw, nodes={fill=pink!20,minimum width=9mm,
 minimum height = 8mm}] 
  {
    \node{0.3}; & \node{0.3}; &  \node {0};  &  \node{$\cdots$}; &  \node {0.3};   \\
    \node{0}; &      \node{0.3}; &   \node {0}; &  \node{$\cdots$};  &       \node {0}; \\
    \node{0.3}; &      \node{0}; &   \node {0.3}; &  \node{$\cdots$};  &       \node {0.3}; \\
    \node{$\vdots$}; &  \node{$\vdots$}; &  \node{$\vdots$}; &  \node{$\ddots$};  &      \node{$\vdots$}; \\
    \node{0.5}; &      \node{0.7}; &   \node {0.5}; &  \node{$\cdots$};  &       \node {0.3}; \\
  };
    \draw [decorate,decoration={brace,amplitude=5pt,mirror}, xshift=0pt, yshift=-0.5pt]
    (-1.39,-1.23) -- (1.39,-1.23) node [black,midway,yshift=-0.6cm] 
{\footnotesize $m_{3}$};
    \draw (0,1.75) node {log of mean (Poisson)};
  \end{scope}


\end{tikzpicture}
\caption{
Algorithm Illustration, where a Question Mark Represents a Missing Value. }\label{TIKZ}
\end{figure}
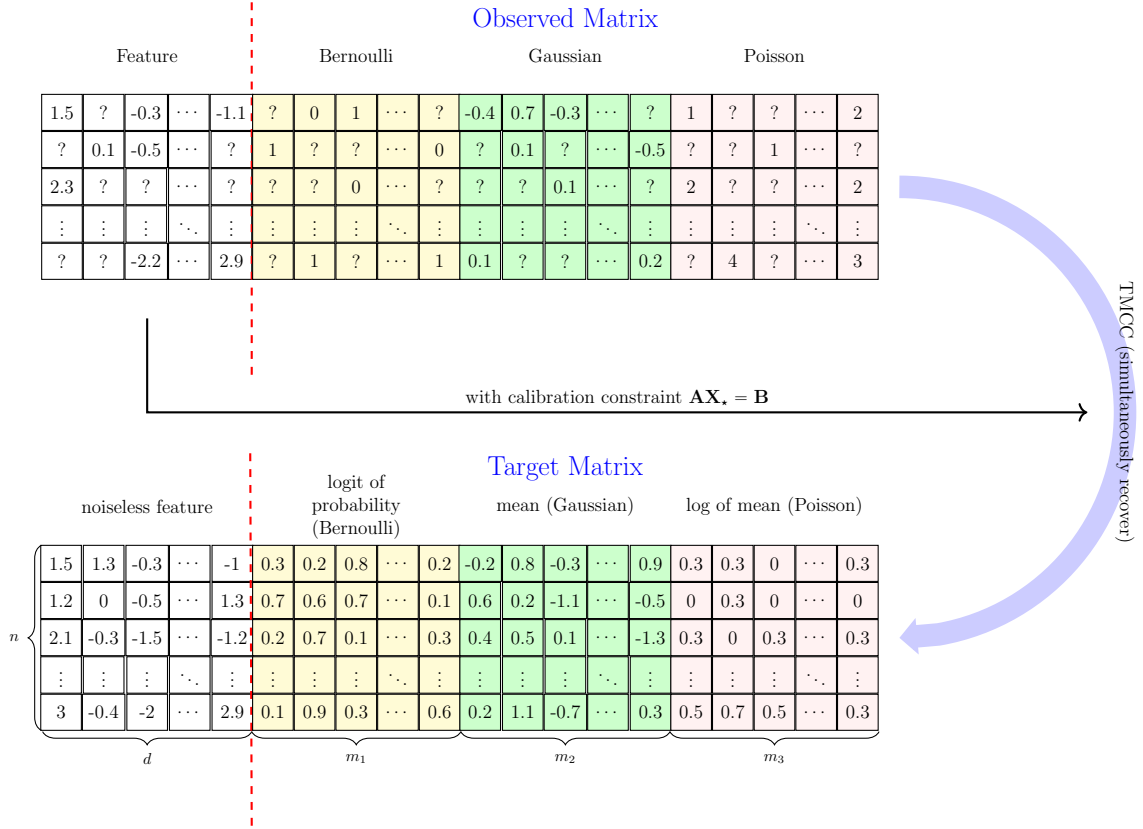

Given density functions $\{f^{(s)}\}_{s=1}^{S}$,  or collection of $\{g^{(s)}\}_{s=1}^{S}$ and $\{h^{(s)}\}_{s=1}^{S}$,  the negative quasi log-likelihood function for $\bZ^{\dagger}$ is 
\begin{align}\label{original likelihood}
\ell_{0}\left(\bZ^{\dagger}\right) =& -\sum_{s=1}^{S} \sum_{(i,j)\in [n] \times [m_{s}]   } r_{y,ij}^{(s)} \log\left\{f^{(s)}(y_{ij}^{(s)}|z_{ij}^{\dagger(s)})\right\}\notag\\
=& \sum_{s=1}^{S} \sum_{(i,j)\in [n] \times [m_{s}]   } r_{y,ij}^{(s)}\left[\left\{-y_{ij}^{(s)}z_{ij}^{\dagger(s)}+g^{(s)}\left(z_{ij}^{\dagger(s)}\right)\right\} + \log\left\{h^{(s)}(y_{ij}^{(s)}) \right\}\right]
\end{align}
where $\bZ^{\dagger} = [\bZ^{\dagger(1)}, \ldots, \bZ^{\dagger(S)}]$ and $\bZ^{\dagger(s)} = (z_{ij}^{\dagger(s)}) \in \mathbb{R}^{n \times m_{s}}$. 
As the argument $\bZ^{\dagger}$ is irrelevant with the  second term  of \eqref{original likelihood}, we refine \eqref{original likelihood} as
\begin{align}
\ell\left(\bZ^{\dagger}\right) = \sum_{s=1}^{S} \sum_{(i,j)\in [n] \times [m_{s}]} r_{y,ij}^{(s)}\left\{-y_{ij}^{(s)}z_{ij}^{\dagger(s)}+g^{(s)}\left(z_{ij}^{\dagger(s)}\right)\right\}.
\notag
\end{align}
Let $\bY = [\bY^{(1)}, \ldots, \bY^{(S)}]$ be the collection of each specific task, $D = d + \sum_{s=1}^{S}m_{s}$ and $\bM_{\star}$ be the concatenated matrix $[\bX_{\star}, \bZ_{\star}]\in\mathbb{R}^{n\times D}$.
Suppose we have an additional calibration constraint $\bA\bX_{\star} = \bB$, where $\bA$ and $\bB$ are available.  To incorporate the calibration constraint, the estimation procedure can be formulated as 
\begin{align}\label{main object}
\wh{\bM} &=  \argmin_{\bM^{\dagger} \in \mathbb{R}^{n \times D}}
 \left[ \frac{1}{nD}\left\{\ell\left(\bZ^{\dagger}\right)  + \frac{1}{2}\Norm{\bR_{x} \circ \left(\bX^{\dagger} - \bX\right)}_{F}^{2} \right\} \right.  \notag \\
& \left. \ \ \ \ \ \ \ \ \ \ \ \ \ \ \ \ \ \ \ \ \quad\quad + \tau_{1}\Norm{\bA\bX^{\dagger} - \bB}_{F}^{2}   + \tau_2\Norm{\bM^{\dagger}}_{\star} \right]\notag \\
&=: \argmin_{\bM^{\dagger} \in \mathbb{R}^{n \times D}} f_{\tau_{1}}\left(\bM^{\dagger}\right)+ \tau_2\Norm{\bM^{\dagger}}_{\star}\notag\\
& =: \argmin_{\bM^{\dagger} \in \mathbb{R}^{n \times D}} \mathcal{L}_{\tau_{1}, \tau_{2}}\left(\bM^{\dagger}\right),
\end{align}
where $\bM^{\dagger} = [\bX^{\dagger}, \bZ^{\dagger}]$, $\wh{\bM} = [\wh{\bX}, \wh{\bZ}]$ and we employ the commonly used square loss to complete the feature part. One trivial case is that when the calibration information is strong enough, i.e., $\bA$ is invertible, the feature matrix  $\bX_{\star}$ can be exactly  recovered.   Note that our learning object is quite general. When $S = 1$, the entries within the only response matrix come from Bernoulli distributions, and there's no third term related to calibration in \eqref{main object}, our learning object degenerates to the case considered in \cite{goldberg2010transduction}, and \cite{xu2018matrix}. However, \cite{goldberg2010transduction} had an additional assumption that there exists a linear relationship between the feature matrix and the hidden parameter matrix. We do not make such a structural assumption. Besides, \cite{xu2018matrix} used additional nuclear norms to penalize the feature matrix for single task learning. When there is no feature information, the objective function will degenerate to the case in \cite{alaya2019collective}.

We propose a \textit{Transductive Matrix Completion with Calibration} (TMCC)  algorithm to obtain the estimator in \eqref{main object}. 
 In the algorithm, the $ij$-th element of the gradient $\partial f_{\tau_{1}}(\bM^{\dagger})$ with respect to $\bM^{\dagger}$ is 
 \begin{align*}
\partial f_{\tau_{1}}(\bM^{\dagger})_{ij}= 
\begin{dcases}
\frac{r_{x,ij}}{nD}\left(x_{ij}^{\dagger} - x_{ij}\right) + 2\tau_1(\bA\trans\bA\bX^{\dagger} - \bA\trans\bB)_{ij},\quad \mbox{for}\ j \in [d],
\\
\frac{r_{y,ij}}{nD}\left\{-y_{ij_{\circ}}^{(s)} + \frac{\partial g^{(s)}(z_{ij_{\circ}}^{\dagger(s)})}{\partial z_{ij_{\circ}}^{\dagger(s)}}  \right\}, \quad\quad\quad\quad\quad\quad\quad \mbox{for}\ \sum_{k=1}^{s-1}m_{k} + d <j\leq \sum_{k=1}^{s}m_{k} + d,
\end{dcases}
\end{align*}
where $j_{\circ} = j- d - \sum_{k=1}^{s-1}m_{k}$. For any generic matrix $\bS$ with the singular value decomposition (SVD) $\bS=\bU\bSigma\bV\trans$, $\bSigma = \mbox{diag}(\sigma_{i})$, $\sigma_{1} \ge \cdots \ge \sigma_{r}$  and $r = \mbox{rank}(\bS)$. Denote the singular value soft threshold operator by $\mathcal{T}_{c}(\bS) = \bU\mbox{diag}((\sigma_{i} - c)_{+})\bV\trans$ for a constant $c > 0$ and $x_{+} = \max(x, 0)$. The detailed algorithm is presented in Algorithm \ref{FPC}.
\begin{algorithm}[!t]
   \caption{TMCC algorithm}
   \label{FPC}
   \KwIn{Incomplete matrices $\bX$, $\bY$; indicator matrices $\bR_x$, $\bR_y$; calibration constraint matrices $\bA$ and $\bB$, tuning parameters $\tau_{1}$, $\tau_{2}$;  learning depth $K$, step size $\eta$, stopping criterion $\kappa$.}
	\kwInit{Random matrices $\bM^{(0)} = \bM^{(1)} \in\mathbb{R}^{n\times D}$,  $c = 1$.}
   \For {$k = 1$ \KwTo $K$  }
   	{ 
    	Compute $\theta = (c-1)/(c+2)$.\\
    	Compute $\bQ = (1+\theta)\bM^{(k)} - \theta \bM^{(k-1)}. $\\
    	Compute $\bT = \bQ - \eta\partial f_{\tau_{1}}(\bQ)$. \\
    	Compute $\bM^{(k+1)} = \mathcal{T}_{\eta\tau_{2}}(\bT)$. \\
   		 \uIf{ $\mathcal{L}_{\tau_{1}, \tau_{2}}(\bM^{(k+1)}) > \mathcal{L}_{\tau_{1}, \tau_{2}}(\bM^{(k)})$ }{
    	$c = 1$;
 		}
 		\Else{
    		$c = c+1$; 
  		}
  		\If{ $ \Abs{\mathcal{L}_{\tau_{1}, \tau_{2}}(\bM^{(k+1)}) -  \mathcal{L}_{\tau_{1}, \tau_{2}}(\bM^{(k)})} \leq \kappa$ }{
    		$\bM^{\ddagger} = \bM^{(k+1)}$;\\
    		break;
 		}
   	}
   \KwOut{$\bM^{\ddagger}$}
\end{algorithm}
The parameter $K$ is a predetermined integer that controls the learning depth of TMCC, and $\eta$ is a predetermined constant for the step size. Within the algorithm, $\kappa$ in the Algorithm \ref{FPC} is a predetermined  stopping criterion. The proposed algorithm iteratively updates the gradient and does singular value thresholding (SVT) \cite{cai2010singular},  in addition to an accelerated proximal gradient decent \cite{ji2009accelerated} step to get a fast convergence rate.

\section{Theoretical Guarantee}

In this section, we first provide convergence rate  analysis for TMCC algorithm. Before that, we make the following technical assumptions.
\begin{assumption}\label{A1}
 There exists a positive constant $ p_{\min}$ such that
 \begin{align} 
    \min \left[ \left\{\min_{s \in [S] }\min_{(i,j) \in [n] \times [m_{s}]   } \pi_{y, ij}^{(s)} \right\} ,
    \left\{   \min_{(i,j) \in [n] \times [d]  } \pi_{x, ij}  \right\} \right] \geq p_{\min},\notag
\end{align}
where $\pi_{x,ij} = \mathbb{P}(r_{x,ij} = 1)$, $\pi_{y,ij}^{(s)} = \mathbb{P}(r_{y,ij}^{(s)} = 1)$.
Further, there exists a positive constant $
\gamma$ such that 
\begin{align}
   &\max \left( \left[ \max_{i \in  [n]} \left\{\pi_{x,i\cdot} + \pi_{y, i\cdot}\right\}\right] ,
   \left[\max_{j \in [d]} \pi_{x, \cdot j}\right], \left[\max_{s\in [S]}\max_{j \in [m_{s}]} \pi_{y, \cdot j}^{(s)} \right]   \right) \leq  \gamma,\notag 
\end{align}
where $ \pi_{x, i\cdot} = \sum_{j\in [d]} \pi_{x,ij}$, $ \pi_{x, \cdot j} = \sum_{i\in [n]} \pi_{x,ij}$, $ \pi_{y, \cdot j}^{(s)} = \sum_{i\in [n]} \pi_{y,ij}^{(s)}$, 
$\pi_{y, i\cdot} =$ $ \sum_{s \in [S]}\sum_{j\in [m_{s}]} \pi_{y,ij}^{(s)}$.\\
\end{assumption}
\begin{assumption}\label{A2}
 There exists a positive constant $\alpha$, such that
$
 \max \left\{ \Norm{\bX_{\star}}_{\infty} , \Norm{\bZ_{\star}}_{\infty}  \right\} \leq \alpha.
$
\end{assumption}
\begin{assumption}\label{A3}
Let $\mathcal{D} = [ -\alpha - \delta, \alpha + \delta]$, for some $\delta > 0$. 
For any $z \in \mathcal{D}$, there exist positive constants $L_{\alpha}$ and $U_{\alpha}$, such that 
$L_{\alpha}  \leq   (g^{(s)})''(z) \leq  U_{\alpha},$
where $g^{(s)}$ is the link function  of exponential family for $s$-th response matrix, for $s \in[S]$.
\end{assumption}
\begin{assumption}\label{A3.5}
There exists a positive constant $\zeta$, for any $\lambda \in \mathbb{R}$, such that  $ \mathbb{E}(e^{\lambda\epsilon_{ij}}) \leq e^{\lambda^{2}\zeta/2} $, where $\epsilon_{ij}$'s are noises for the feature matrix. 
\end{assumption}
\begin{assumption}\label{A4}
There exists a constant $C>0$, such that $\sigma_{\min}(\bA) \geq C > 0$.
\end{assumption}

Assumption~\ref{A1} controls the sampling probabilities for our model. The first part ensures that all the sampling probabilities are bounded away from zero. The second part aims to bound the operator norm of a Rademacher matrix of the same dimension as $\bM_{\star}$ stochastically.  
Although we assume that the parameter matrix $\bZ_{\star}$  is  bounded in Assumption~\ref{A2}, the support of entries in $\bY$ can be unbounded. For instance, the support of a Poisson distribution is unbounded, while its mean is fixed.  In other words,   Assumption \ref{A2} implies  $\bM_{\star} \in \mathcal{G}_{\infty}(\alpha) := \{\bG \in \mathbb{R}^{n \times D}: \norm{\bG}_{\infty} \leq \alpha\}$.
Assumption \ref{A3} is mild under the canonical exponential family framework. That is, we have $\mbox{Var}(y_{ij}^{s}|g^{s}(z_{\star,ij}^{s})) = (g^{(s)})''(z_{\star,ij}^{s}) > 0$ for each $i,j$. We extend  it a little bit with tolerance $\delta$   for ease of proof. Furthermore, define $\tilde{L}_{\alpha} = (L_{\alpha} \wedge 1)$ and $\tilde{U}_{\alpha} = ( U_{\alpha} \vee \zeta \vee 1)$.
Assumption \ref{A3.5} implies that the errors for the feature matrix come from sub-Gaussian distributions.
Assumption \ref{A4} indicates that the calibration matrix $\bA$ should be of full rank. 
The following Theorem \ref{algorithm convergence} shows the convergence rate of the proposed algorithm is $O(1/k^{2})$. 
\begin{theorem}\label{algorithm convergence}
Suppose that Assumption $\ref{A1}\sim\ref{A3.5}$ hold, $\tau_{1} \leq \ \ \ $ $c_1 \min[ \{\sigma_{\max}(\bA)\}^{-2}(nD\alpha)^{-1/2}   $,
$  \{ nD\sigma_{\min}^{2}(\bA) \}^{-1}
  \tilde{L}_{\alpha}p_{\min} ] $ and  $\tau_{2} = (nD)^{-1}  2c_{2}\{ (\tilde{U}_{\alpha})^{1/2} \vee 1/\delta \}
\{ \gamma^{1/2} +   (\log ( n\vee D  )    ) ^{3/2} \} $.   The sequences $\{\bM^{(k)}\}$ generated by Algorithm \ref{FPC} satisfy
\begin{align}
  f_{\tau_{1}}\left(\bM^{(k)}  \right)- 
  f_{\tau_{1}}\left(\bM_{\star}\right) 
  \leq 
  \frac{  2 \tilde{L}\Norm{\bM^{(0)} - \bM_{\star}   }_{F}^{2}    }{\eta (k+1)^{2} },\notag
\end{align}
with probability at least $1 - 4/(n+D)$, where 
$\tilde{L}$ is a constant related to $c_{1}$ and  $\alpha^{1/2}\tilde{U}_{\alpha}$. 
\end{theorem}

The following theorem presents the statistical guarantee for the proposed method.

\begin{theorem}\label{T1}
Suppose that Assumption $\ref{A1} \sim \ref{A3.5}$ hold, $\tau_{1} \leq \ \ $
$\{nD\sigma_{\min}^{2}(\bA)\}^{-1} c_{1} \tilde{L}_{\alpha}p_{\min}$ and  $\tau_{2} = (nD)^{-1}  2c_{2}\{ (\tilde{U}_{\alpha})^{1/2} \vee 1/\delta \} 
\{ \gamma^{1/2} +   (\log ( n\vee D  )    ) ^{3/2} \} $. Then, with probability as least $1 - 4/(n + D)$,
\begin{align}\label{C1_eq}
&\quad \left\{\frac{1}{nD} + \frac{8\tau_{1}}{\tilde{L}_{\alpha}p_{\min}  }\sigma_{\min}^{2}(\bA)  \right\} \Norm{\wh{\bX} - \bX_{\star}}_{F}^{2}     +        \frac{1}{nD}\Norm{\wh{\bZ} - \bZ_{\star}}_{F}^{2}\notag \\  
&\leq  \frac{c {\text{rank}} (\bM_{\star})    }{nDp_{\min}^{2}}\left\{ \alpha^{2} + \frac{  \tilde{U}_{\alpha} \vee  1/\delta^{2}     }{  \tilde{L}_{\alpha}^{2}   }       \right\}
\left\{  \gamma + \log^{3}(n \vee D)    \right\},
\end{align} 
where $c, c_{1}, c_{2}$ are positive constants.
\end{theorem}

Theorem~\ref{T1} implies that our recovered error for the target matrix, in the sense of squared Frobenius norm, is bounded by the right hand side of \eqref{C1_eq}, with probability approaching $1$ when $n$ and $D$ are large enough. 
Suppose the feature matrix is also regarded as a response matrix from Gaussian noise with unit variance, by directly applying the Theorem 7 in \cite{alaya2019collective}, we have $(nD)^{-1}(\norm{\wh{\bX} - \bX_{\star}}_{F}^{2}     +        \norm{\wh{\bZ} - \bZ_{\star}}_{F}^{2} )$ less than the terms in the second line of \eqref{C1_eq}. 
Fortunately, with the help of calibration information, we have a constant order improvement.  Specifically, together with Assumption \ref{A4}, we have
\begin{align}
   &\frac{1}{nD} \Norm{\wh{\bX} - \bX_{\star}}_{F}^{2}     +        \frac{1}{nD}\Norm{\wh{\bZ} - \bZ_{\star}}_{F}^{2} \notag\\
  \stackrel{(i)}{<}&
 \left\{\frac{1}{nD} + \frac{8\tau_{1}}{\tilde{L}_{\alpha}p_{\min}  }\sigma_{\min}^{2}(\bA)  \right\} \Norm{\wh{\bX} - \bX_{\star}}_{F}^{2}     +        \frac{1}{nD}\Norm{\wh{\bZ} - \bZ_{\star}}_{F}^{2} , \notag
\end{align}
where the inequality $(i)$ is strict, which is one of the main theoretical contributions of this paper.

\section{Experiments}

We conduct a simulation study to illustrate the performance of the  proposed method. Let $\bP \in \mathbb{R}^{n \times r}$ and $\bQ \in \mathbb{R}^{d \times r}$ whose entries are generated independently from a uniform distribution  over $(0,1)$. Then, let $\bX_{\star}^{0} = \bP \bQ\trans$ and $\bX_{\star} = \bX_{\star}^{0} / \norm{ \bX_{\star}^{0}  }_{\infty}$. 
We further generate coefficient matrix $\bW^{(s)} \in \mathbb{R}^{d \times m_{s}}$ whose elements are independently generated from a uniform distribution over $(0,1)$ for $s = 1, 2, 3$.  Let $\wt{\bZ}_{\star}^{(s)} = \bX_{\star}\bW^{(s)}$. We call this setting as ``linear'' case. By normalization, we have $\bZ_{\star}^{(s)} = \wt{\bZ}_{\star}^{(s)} / \norm{\wt{\bZ}_{\star}^{(s)}}_{\infty}$.
Specifically, $\bY^{(1)}$ has Bernoulli entries with support $\{0, 1\}$, $\bY^{(2)}$ has Poisson entries and $\bY^{(3)}$ has Gaussian entries with known $\sigma^{2} = 1$. All the link between $\bY^{(s)}$ and ${\bZ^{(s)}}$ are the same as those in Section 2. For calibration information, suppose we know the column means for $\bX_{\star}$, i.e., $\bA = (1/n)\bone_{1\times n}$ and $\bB = \bA\bX_{\star}$. On the other hand, we have a ``nonlinear'' case, i.e., we assume that $\wt{\bZ}_{\star}^{(s)}$ is generated by an element-wise nonlinear transformation of $\bX_{\star}$. Specifically, let $z_{\star,ij}^{(s)} = t^{(s)}(x_{\star,ij})$, where $t^{(1)}(x) = x^2 +x +0.5$, $t^{(2)}(x) = -x^2 -x$, $t^{(3)}(x)= -x^2 - 2x +0.2$. The normalization procedure is the same as the ``linear'' case. The proposed method TMCC is compared  with three other approaches.

1. \textbf{MC\_0}: Exactly the same as modified TMCC except for the gradient updating procedure. No calibration information is considered. Therefore in $\partial f_{\tau_{1}}(\bM^{\dagger})_{ij}$
the term $2\tau_{1}(\bA\trans\bA\bX^{\dagger} - \bA\trans\bB)_{ij}$ is eliminated. 

2. \textbf{CMC\_SI}: Collective matrix completion (\textbf{CMC}) \cite{alaya2019collective} is used to complete the parameters for the target matrix, and Soft-Impute (\textbf{SI}) method from \cite{mazumder2010spectral} is used to complete the feature matrix separately.

3. \textbf{TS}: A two-stage method, where, at the first stage, only the feature matrix is imputed by the Soft-Impute method, and at the second stage, the method \textbf{MC\_0} is applied to the concatenated matrix joined by the feature matrix and the observed response matrices.

Specifically, MC\_0 and TMCC share the same strategy, i.e. simultaneously recovering all matrices, while CMC\_SI chooses to recover separately and TS opts to recover step by step.

In the experiments, we set  $n = 1500$, $d = 500$, $m_{1} = m_{2} = m_{3} = 500$ and choose learning depth $K = 1000$ and stopping criterion $\kappa = 10^{-7}$.  Further, we compare different methods with  rank $r \in \{5, 10, 15\}$. The missing rate $\nu \in \{60\%, 80\%\}$ of both the feature matrix and response matrix are the same in each experiment. For TMCC, we tune  $\tau_1$ and $\tau_2$ on one independent validation set and apply the same parameters to all other repeated 50 simulations. Further, other compared methods employ the same procedure as TMCC while they only have to be tuned for the parameter $\tau_2$.

The performance of each method is evaluated via the mean value and standard deviation of the  relative errors (RE) based on repeated experiments. Specifically,   the relative error of a recovered feature matrix is $\mbox{RE}(\wh{\bX}) = \norm{\wh{\bX} - \bX_{\star} }_{F}/\norm{\bX_{\star}}_{F}$ and that of target matrix  $\mbox{RE}(\wh{\bZ}) = \norm{\wh{\bZ} - \bZ_{\star} }_{F}/\norm{\bZ_{\star}}_{F}$. Experiment results are summarized in Fig~\ref{fig_feature} and Fig~\ref{fig_target}.

\begin{figure}[h!]
     \centering
    \includegraphics[width = 15cm]{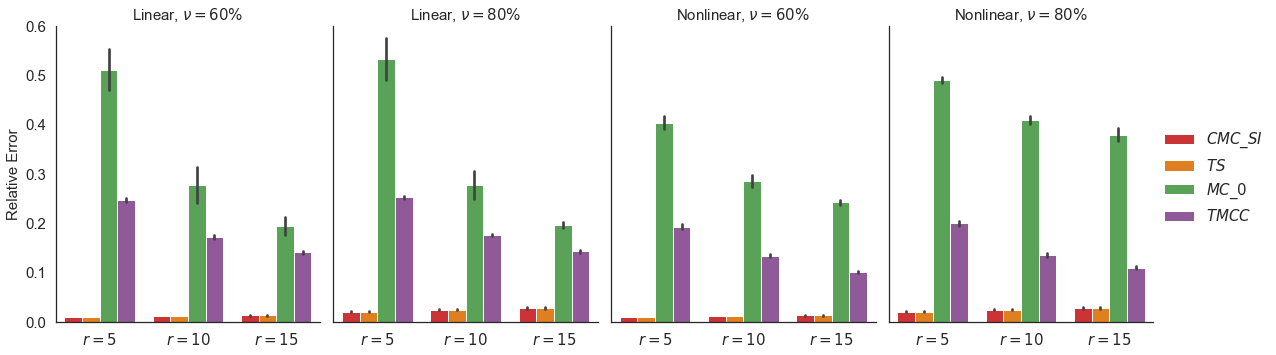}
    \caption{Relative Error of Feature Matrix (with the Black Lines Representing ± the Standard Error)} \label{fig_feature}
\end{figure}

\begin{figure}[h!]
     \centering
    \includegraphics[width = 15cm]{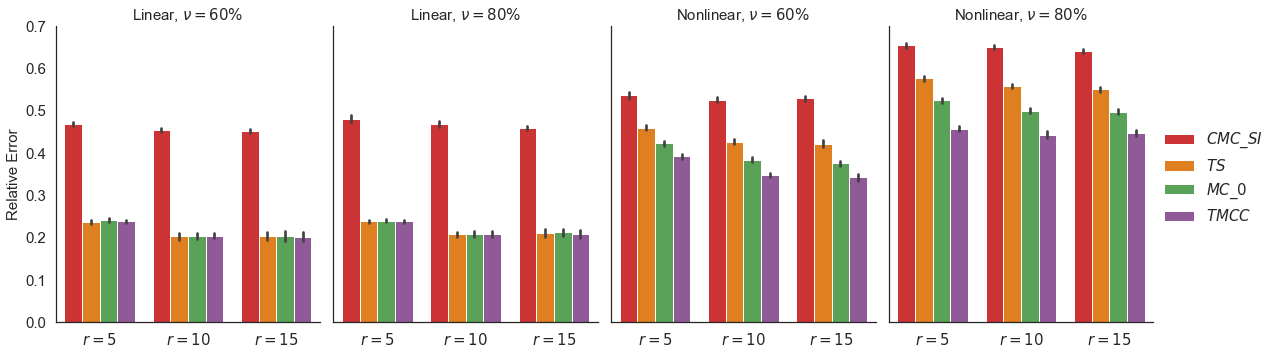}
    \caption{Relative Error of Target Matrix (with the Black Lines Representing ± the Standard Error)} \label{fig_target}
\end{figure}

In Fig~\ref{fig_feature}, it is an unsurprising fact that  TMCC surpasses MC\_0 with the help of calibration information. For instance, in the linear case with  $\nu = 60\%$, the mean of RE of $\hat\bX$ by MC\_0 is 0.51 with standard error (SE) 0.0417 when $r = 5$ while that by TMCC is 0.25 with SE 0.0040. TMCC's RE is less than half of MC\_0's and situations in other cases are alike. Besides, CMC\_SI and TS perform the best in all scenarios. Specifically, the mean of RE of $\hat\bX$ by CMC\_SI is 0.01 with SE 0.0005, and that by TS is 0.01 with SE 0.0006 when $r = 10$ in the nonlinear case with $\nu = 60\%$.  
It is because both CMC\_SI and TS complete the feature matrix without considering the target matrix. However, our primary goal is to recover $\bZ_{\star}$, i.e., achieving a low relative error of $\hat{\bZ}$. In Fig~\ref{fig_target}, the far lower relative error of $\hat{\bX}$  by CMC\_SI and TS does not bring a lower relative error of $\hat{\bZ}$. Of all the four approaches, CMC\_SI is the only one that  fails to  take advantage of the feature matrix to recover the target matrix. That is why it attains the greatest relative error of the target matrix.
For example, its mean of RE of $\hat\bZ$ is 0.46 with SE 0.0046 when $r= 15$ in the linear case with $\nu = 80\%$ and 0.53 with SE 0.0057 when $r=15$ in the nonlinear case with $\nu = 60\%$, more than twice of the other three methods.
Besides,
TS, MC\_0, and TMCC behave similarly in the linear case while they display great differences in the nonlinear case. When the relationship between $\bX_{\star}$ and $\bZ_{\star}$ is not linear, simultaneously recovering demonstrates great strength compared with recovering step by step and recovering  separately. 
In the nonlinear case with $\nu = 80\%$, the means of RE of $\hat\bZ$  by CMC\_SI, TS, MC\_0, and TMCC are 0.66, 0.58, 0.52, and 0.46 respectively when $r = 5$. Situations in other cases are almost the same.
What is noteworthy is that TMCC also overtakes MC\_0 with respect to target matrices, which implies the power of calibration information again. Overall, relatively low standard errors indicate the stability of these algorithms.

\section{Conclusion}

We proposed a statistical framework  for multi-task learning under exponential family matrix completion framework with known calibration information. Statistical guarantee   of our proposed estimator has been shown and a constant order improvement is achieved compared with existing methods.  We have also shown that the proposed algorithm has a convergence rate of $O(1/k^{2})$. The simulation study also shows that our proposed method has numerical benefits.

\bibliographystyle{IEEEbib}
\bibliography{ref_2}

\newpage
\onecolumn
\appendix

\section*{Appendix}
In this appendix, we provide technical proofs and computational time comparisons for our main article. Specifically, it is organized as follows:
\begin{itemize}
	\item Section \ref{SectionA} introduces necessary preliminary;
	\item Section \ref{SectionB} presents proof of Theorem \ref{algorithm convergence};
	\item Section \ref{SectionC} presents proof of Theorem \ref{T1};
    \item Section \ref{SectionD} provides auxiliary lemmas;
    \item Section \ref{SectionE} provides computational time comparison for the simulations.
\end{itemize}

\renewcommand{\theequation}{\thesection.\arabic{equation}}

\section{PRELIMINARY}\label{SectionA}

In the following, we investigate the theoretical properties of the proposed \textit{TMCC} algorithm. Before going through those theoretical results, we first introduce some useful matrix norm inequalities. For two matrices $\bS$ and $\bT$ with the same dimension, define the inner product  in terms of the trace of the matrix product, i.e., $\langle \bS, \bT \rangle = \mbox{tr}(\bS\trans\bT)$. We have trace duality inequality
\begin{align}\label{trace inequality}
  \Abs{ \mbox{tr}(\bS\trans\bT)} \leq \Norm{\bS} \Norm{\bT}_{\star},
\end{align}
and bound for nuclear norm
\begin{align}\label{nuclear and F}
  \Norm{\bS}_{\star} \leq \sqrt{\mbox{rank}(\bS)}\Norm{\bS}_{F}.
\end{align}
Further, 
\begin{align}\label{F norm inequality}
 \sigma_{\min}(\bS) \Norm{\bT}_{F}\leq \Norm{\bS\bT}_{F},
\end{align}
where the proof of \eqref{F norm inequality} is presented in Lemma \ref{lambda min}. Suppose the SVD of $\bS$ is $\bU_{\bS}\Sigma_{\bS}\bV\trans_{\bS}$.
Let $\wt{\bP}_\bS = \bS(\bS\trans\bS)^{-1}\bS\trans $ be the projection matrix based on $\bS$.  Define 
$\mathcal{P}_{\bS}^{\perp}(\bT) = \wt{\bP}_{\bU_{\bS}^{\perp}}\bT \wt{\bP}_{\bV_{\bS}^{\perp}}$, where
$\bF^{\perp} = (\bI - \wt{\bP}_{ \bF })\bF$ for a matrix $\bF$. Notice that $\mathcal{P}_{\bS}^{\perp}(\bT) $ is not a projection since it is not idempotent. 
 Let $\mathcal{P}_{\bS}(\bT) = \bT - \mathcal{P}_{\bS}^{\perp}(\bT)$ and we have
\begin{align}\label{projection rank}
 \mbox{rank} \left\{\mathcal{P}_{\bS}(\bT)\right\} \leq 2\mbox{rank}(\bS).
\end{align}
By \cite{klopp2014noisy}, we have
\begin{align}\label{nuclear difference}
 \Norm{\bS}_{\star} - \Norm{\bT}_{\star} \leq \Norm{ \mathcal{P}_{\bS}(\bS - \bT)  }_{\star} + \Norm{ \mathcal{P}_{\bS}^{\perp}(\bS - \bT)  }_{\star}.
\end{align}
Finally, we define the Bregman divergence.
\begin{definition}
Let $\mathcal{C}$ be a closed set. For a continuously-differentiable function $F : \mathcal{C} \rightarrow \mathbb{R}$, the Bregman divergence associated with $F$ between $x, y \in \mathcal{C}$ is
\begin{align}
d_{F}(x, y) = F(x) - F(y) - 
\langle   \nabla F(y), x - y   \rangle. \notag
\end{align}
\end{definition}
Mathematically, $d_{F}(x,y)$ is equivalent to the first-order Taylor expansion error of $F(x)$ evaluated at $y$.

Denote $\bSigma_{R}$ to be an  $n \times D$ matrix, consisting of $i.i.d.$ Rademacher sequences $\{\xi_{x,ij}:i\in [n], j \in [d]\}$  and $\{\xi_{y,ij}^{(s)}: i \in [n], j\in [m_{s}]\}$ for $s \in [S]$.  Specifically,
\begin{align}
  \bSigma_{R} =&\sum_{(i,j)}^{(n, d)}\xi_{x,ij} r_{x,ij}\bE_{x,ij} +\sum_{s=1}^{S} \sum_{(i,j)}^{(n, m_{s})} \xi_{y,ij}^{(s)} r_{y,ij}^{(s)}\bE_{y,ij}^{(s)},\notag
\end{align}
where $\bE_{x,ij} = \be_{i}(n) \be_{j}(D)\trans$ and $\bE_{y,ij}^{(s)} =  \be_{i}(n) \be_{\tilde{j}}(D)\trans$ are $\mathbb{R}^{n\times D}$ matrices lying in the set of canonical bases, $\be_{i}(l)$ is the $j$-th unit vector of length $l$ and $\tilde{j} = j + d + \sum_{t=1}^{s-1} m_{t} $.

\section{PROOF OF THEOREM 1}\label{SectionB}

\begin{proof}
 Given two arbitrary matrices $\bS$, $\bT \in \mathbb{R}^{n \times D}$, we have 

\begin{align}
&\Abs{f_{\tau_{1}}(\bS) - 
f_{\tau_{1}}(\bT)} \notag\\
=& 
\left\vert\left\{\tilde{\ell}(\bS) - \tilde{\ell}(\bT)\right\} -
\tau_{1} \left\{ \Norm{\bA\left( \bS - \bM_{\star}\right)}_{ F}^{2} - 
 \Norm{\bA\left(\bT - \bM_{\star}\right)}_{ F}^{2} \right\} \right\vert \notag\\
\leq& \Abs{\tilde{\ell}(\bS) - \tilde{\ell}(\bT)} + 
\tau_{1} \Abs{\left\{ \Norm{\bA\left( \bS - \bM_{\star}\right)}_{ F}^{2} - \Norm{\bA\left(\bT - \bM_{\star}\right)}_{ F}^{2}
  \right\}
} \notag \\
\leq& \Norm{\nabla \tilde{\ell}\left(\wt{\bM}\right)} \Norm{\bS-\bT}_{\star}
+  \tau_{1}\sigma_{1}^{2}(\bA)\Norm{\bS - \bT}_{ F} \Norm{\bS + \bT - 2\bM_{\star}}_{ F}\notag \\
\leq& \left\{  \sqrt{n \wedge D}\Norm{\nabla \tilde{\ell}\left(\wt{\bM}\right)}
 + 2 \tau_{1}\sigma_{1}^{2}(\bA)\sqrt{nD\alpha}
 \right\}
 \Norm{\bS - \bT}_{ F},\notag
\end{align}

where $\wt{\bM} \in \mathcal{G}_{\infty}(\alpha)$. We focus on bound of $\norm{\nabla \tilde{\ell}(\wt{\bM})  }$. 
\begin{align}
&\Norm{\nabla \tilde{\ell}\left(\wt{\bM}\right)}\notag\\
 \leq &\Norm{\nabla \tilde{\ell}(\bM_{\star})} +
\Norm{ \nabla \tilde{\ell}(\wt{\bM}) - \nabla \tilde{\ell}(\bM_{\star})}\notag \\
\leq& \Norm{\nabla \tilde{\ell}(\bM_{\star})}+
\frac{1}{nD}   \left\Vert \sum_{s=1}^{S}\sum_{(i,j) \in [n] \times [m_{s}]}r_{y,ij}^{(s)}\bE_{y,ij}^{(s)} \times \left\{(g^{(s)})'(\tilde{z}_{\star,ij}^{(s)}) - (g^{(s)})'(z_{\star,ij}^{(s)})\right\} \right.\notag\\
& 
\left.\quad\quad\quad\quad\quad\quad\quad\quad\quad\quad\quad\quad
+\sum_{(i,j) \in [n] \times [d]} r_{x,ij}\bE_{x,ij}\left\{
 g_{N}'(\tilde{x}_{ij}) - g_{N}'(x_{ij}) \right\}
\right\Vert
\notag \\
\leq& \Norm{\nabla \tilde{\ell}(\bM_{\star})} + 
\frac{\tilde{U}_{\alpha}}{nD} \left\Vert  \sum_{s=1}^{S}\sum_{(i,j) \in [n] \times [m_{s}]}r_{y,ij}^{(s)}\bE_{y,ij}^{(s)} \Abs{\tilde{z}_{\star,ij}^{(s)} - z_{\star,ij}^{(s)}}
+
\sum_{(i,j) \in [n] \times [d]} r_{x,ij}\bE_{x,ij}
 \Abs{\tilde{x}_{ij} - x_{ij} }
\right\Vert \notag \\
\leq& \frac{ c_{2}(\sqrt{\tilde{U}_{\alpha}} \vee 1/\delta )\left[ \sqrt{\gamma}  + \left\{\log (n\vee D) \right\}^{3/2} \right]       }{nD} +
\frac{2\sqrt{nD\alpha}\tilde{U}_{\alpha}}{nD}\notag \\
\leq& \frac{c_{3}}{\sqrt{nD}}.\notag
\end{align}
Therefore, as long as 
$\tau_{1} \leq  c_{1}\{\sigma_{1}(\bA)\}^{-2}(nD\alpha)^{-1/2}$, there exists positive constant $\tilde{L} = 2c_{1} + c_3$ such that
$f_{\tau_{1}}(\cdot)$ is $\tilde{L}$-Lipchitz. The remaining proof can be obtained by followed the proof of Theorem 4.1 in \cite{ji2009accelerated}.
\end{proof}

\section{PROOF OF THEOREM 2} \label{SectionC}

\begin{proof}
In general, we use the connection between exponential family distributions and Bregman divergence to argue the basic inequality implied by our estimation procedure. Then, a key quantity can be bounded  and the results can be obtained by the standard argument for matrix completion. Specifically, by basic inequality, we have 
\begin{align}
 &\frac{1}{nD} \left[\sum_{s=1}^{S} \sum_{(i,j) \in [n] \times [m_{s}]} r_{y,ij}^{(s)} \left\{   - y_{ij}^{(s)}\hat{z}_{ij}^{(s)} + g^{(s)}(\hat{z}_{ij}^{(s)}  )   \right\} +  
  \sum_{(i,j) \in [n] \times [d]}   \frac{r_{x,ij} }{2} (\hat{x}_{ij} -x_{ij})^{2}   \right] \notag\\
  &\quad\quad +
 \tau_{1}\Norm{\bA\wh{\bX} - \bB   }_{F}^{2}+
  \tau_{2} \Norm{  \wh{\bM}}_{\star} \notag \\
 \leq& \frac{1}{nD} \left[\sum_{s=1}^{S} \sum_{(i,j) \in [n] \times [m_{s}]} r_{y,ij}^{(s)} \left\{   - y_{ij} ^{(s)}z_{\star,ij} + g^{(s)}(z_{\star,ij}^{(s)}  )   \right\}+  
  \sum_{(i,j) \in [n] \times [d]}   \frac{r_{x,ij}}{2} (x_{\star,ij} -x_{ij})^{2}       \right] \notag\\
  &\quad\quad +\tau_{1} \Norm{\bA\bX_{\star} - \bB   }_{F}^{2}  + \tau_{2} \Norm{  \bM_{\star}}_{\star} \notag \\
  =& \frac{1}{nD} \left[\sum_{s=1}^{S} \sum_{(i,j) \in [n] \times [m_{s}]} r_{y,ij}^{(s)} \left\{   - y_{ij} ^{(s)}z_{\star,ij} + g^{(s)}(z_{\star,ij}^{(s)}  )   \right\}+
    \sum_{(i,j) \in [n] \times [d]}  \frac{r_{x,ij}}{2} (x_{\star,ij} -x_{ij})^{2}   \right] + \tau_{2} \Norm{  \bM_{\star} }_{\star},\notag
   \end{align}
   where the calibration information implies the last equality. Let $g_{N} (x) = x^{2}/2$, expanding quadratic terms of $x_{ij}$,  and we have
 \begin{align}
 &\frac{1}{nD} \left[\sum_{s=1}^{S} \sum_{(i,j) \in [n] \times [m_{s}]} r_{y,ij}^{(s)} \left\{   - y_{ij} ^{(s)}\hat{z}_{ij}^{(s)} + g^{(s)}(\hat{z}_{ij}^{(s)}  )   \right\} + 
  \sum_{(i,j) \in [n] \times [d]}  r_{x,ij}   \left\{   - x_{ij}\hat{x}_{ij} + g_{N}(\hat{x}_{ij})   \right\}  \right] 
 \notag\\
 &\quad\quad+
 \tau_{1}\Norm{\bA\wh{\bX} - \bB   }_{F}^{2} + \tau_{2} \Norm{  \wh{\bM}}_{\star} \notag \\
 \leq&
 \frac{1}{nD} \left[\sum_{s=1}^{S} \sum_{(i,j) \in [n] \times [m_{s}]} r_{y,ij}^{(s)} \left\{   - y_{ij} ^{(s)}z_{\star,ij} + g^{(s)}(z_{\star,ij}^{(s)}  )   \right\} + \sum_{(i,j) \in [n] \times [d]}  r_{x,ij} \left\{   - x_{ij}{x}_{\star,ij} + g_{N}(x_{\star,ij})   \right\} \right]+ \tau_{2} \Norm{  \bM_{\star}  }_{\star}.\notag
   \end{align}
 Rearranging the terms, we obtain
\begin{align}\label{transition_T1}
&\frac{1}{nD} \left( \sum_{s=1}^{S} \sum_{(i,j) \in [n] \times [m_{s}]} r_{y,ij}^{(s)} \left[  -y_{ij}^{(s)}\left(  \wh{z}_{ij}^{(s)} -    z_{\star,ij}^{(s)} \right) +  \left\{g^{(s)}(\wh{z}_{ij}^{(s)}) -   g^{(s)}(z_{\star,ij}^{(s)}) \right\}   \right] +   \right.\notag \\
&\left.  \quad  \sum_{(i,j) \in [n] \times [d]} r_{x,ij}\times \left[ -x_{ij}(\hat{x}_{ij} - x_{\star,ij}) 
+ \left\{ g_{N}(\hat{x}_{ij}) - g_{N}(x_{\star, ij})      \right\}            \right]    \right) \notag \\
\leq& -\tau_{1} \Norm{\bA\wh{\bX} - \bB}_{F}^{2} + \tau_{2} \left(   \Norm{\bM_{\star} }_{\star} - \Norm{\wh{\bM}}_{\star}   \right) \notag \\
=& -\tau_{1} \Norm{\bA  \left(\wh{\bX} - \bX_{\star}\right) }_{F}^{2} + \tau_{2} \left(   \Norm{\bM_{\star} }_{\star} - \Norm{\wh{\bM}}_{\star}   \right).
\end{align}
Plug in the Bregman divergence  into \eqref{transition_T1}  with \eqref{trace inequality}, and we have
\begin{align} \label{T1_transit_0}
&\frac{1}{nD} \left[ \sum_{s=1}^{S} \sum_{(i,j) \in [n] \times [m_{s}]}\left\{r_{y,ij}^{(s)} d_{g^{(s)}}\left(\wh{ z}_{ij}^{(s)}, z_{\star,ij}^{(s)} \right)\right\} + 
 \sum_{(i,j) \in [n] \times [d]}\left\{r_{x, ij}^{(s)} d_{g_{N}}\left(\wh{ x}_{ij} ,  x_{\star,ij} \right)\right\}
\right] \notag \\
\leq&   -\tau_{1} \Norm{\bA\wh{\bX} - \bB}_{F}^{2} +  \tau_{2} \left(   \Norm{\bM_{\star}}_{\star} - \Norm{\wh{\bM}}_{\star}   \right) -\notag \\
&  \quad \frac{1}{nD} \left[ \sum_{s=1}^{S} \sum_{(i,j) \in [n] \times [m_{s}]}r_{y,ij}^{(s)} \times \quad\left\{ \left(g^{(s)}\right)'(z_{\star,ij}^{(s)}) - y_{ij}^{(s)} \right\}
 \left( \wh{z}_{ij}^{(s)} - z_{\star,ij}^{(s)} \right)  + 
  \right. \notag\\
  &\left. \quad\quad\quad \sum_{(i,j) \in [n] \times [d]}r_{x,ij}\left\{ (g_{N})'(x_{\star,ij}) - x_{ij}\right\}  (\hat{x}_{ij} - x_{\star,ij} )       \right]\notag \\
 =& -\tau_{1} \Norm{\bA\wh{\bX} - \bB}_{F}^{2} +  \tau_{2} \left(   \Norm{\bM_{\star}}_{\star} - \Norm{\wh{\bM}}_{\star}   \right) -
 \langle \nabla \tilde{\ell}(\bM_{\star}), \wh{\bM} - \bM_{\star} \rangle \notag \\
 \leq& -\tau_{1} \Norm{\bA\wh{\bX} - \bB}_{F}^{2} +  \tau_{2} \left(   \Norm{\bM_{\star}}_{\star} - \Norm{\wh{\bM}}_{\star}   \right) + \Norm{\nabla \tilde{\ell}(\bM_{\star})}\Norm{ \wh{\bM} - \bM_{\star} }_{\star},
\end{align}
where $\tilde{\ell}(\bM_{\star}) =  (nD)^{-1} [  \ell(\bZ_{\star}) + \sum_{(i,j) \in [n] \times [d]}r_{x,ij}\times \{ -x_{ij}x_{\star,ij} + g_{N}(x_{\star,ij})\} ]$.
With an additional assumption that  $\tau_{2} \geq 2 \norm{\nabla \tilde{\ell}(\bM_{\star})}$, together with  \eqref{nuclear and F} and \eqref{projection rank}, \eqref{T1_transit_0} yields
\begin{align}\label{Bregman Result}
 &\frac{1}{nD} \left\{ \sum_{s=1}^{S} \sum_{(i,j) \in [n] \times [m_{s}]}r_{y, ij}^{(s)} d_{g^{(s)}}\left(\wh{ z}_{ij}^{(s)} ,  z_{\star,ij}^{(s)} \right) + \sum_{(i,j) \in [n] \times [d]} r_{x,ij} d_{g_{N}}(\hat{x}_{ij} , x_{\star,ij})  \right\} \notag \\
 \leq& \frac{3\tau_{2}}{2} \Norm{\mathcal{P}_{\bM} \left(\wh{\bM} - \bM_{\star} \right) }_{\star} - \tau_{1} \Norm{\bA\left(  \wh{\bX} - \bX_{\star}   \right)}_{F}^{2}
\notag \\
\leq& \frac{3\tau_{2}}{2} \sqrt{  2\mbox{rank}(\bM_{\star}) } \Norm{ \wh{\bM} - \bM_{\star}  }_{F} - \tau_{1} \Norm{\bA\left(  \wh{\bX} - \bX_{\star}   \right)}_{F}^{2}.
\end{align}
Meanwhile, as $L_{\alpha} (x-y)^{2} \leq 2 d_{g^{(s)}} (x,y) \leq U_{\alpha}(x-y)^{2}$ and $2d_{g_{N}}(x,y) = (x-y)^{2}$, we have 
\begin{align}
 \Delta_{\tilde{\ell}}^{2} \leq &\frac{2}{\tilde{L}_{\alpha}nD} \left\{ \sum_{s=1}^{S}\sum_{(i,j) \in [n] \times [m_{s}]}r_{y,ij}^{(s)}d_{g^{(s)}}\left(\wh{z}_{ij}^{(s)}, z_{\star,ij}^{(s)}\right)
   +  \sum_{(i,j) \in [n] \times [d]}r_{x,ij}d_{g_{N}}(\hat{x}_{ij},  x_{\star,ij})  \right\},\notag
\end{align}
where 
\begin{align}\label{MSE}
 \Delta_{\tilde{\ell}}^{2} =& \frac{1}{nD} \left\{ \sum_{s=1}^{S}\sum_{(i,j) \in [n] \times [m_{s}]}r_{y,ij}^{(s)} \left(\wh{z}_{ij}^{(s)} -  z_{\star,ij}^{(s)}\right)^{2}
  + \sum_{(i,j) \in [n] \times [d]}r_{x,ij}\left( \hat{x}_{ij} - x_{\star,ij}  \right)^{2}  \right\}.
\end{align}
Put \eqref{Bregman Result} and \eqref{MSE} together, and we obtain
\begin{align}\label{key}
\Delta_{\tilde{\ell}}^{2} \leq& \frac{2}{\tilde{L}_{\alpha}} \left\{  \frac{3\tau_{2}}{2} \sqrt{  2\mbox{rank}(\bM_{\star})  } \Norm{ \wh{\bM} - \bM_{\star}  }_{F}  -\tau_{1} \Norm{\bA\left(  \wh{\bX} - \bX_{\star}   \right)}_{F}^{2}    \right\}.
\end{align}

With Lemma \ref{nuclear} and \ref{Bound} in the Appendix, follow the proof of Theorem 3 in \cite{alaya2019collective}, our main theorem is proved.

\end{proof}

\section{AUXILIARY LEMMAS}\label{SectionD}
\begin{lemma}\label{nuclear}
Let  arbitrary matrices $ \bS, \bT \in \mathcal{G}_{\infty}(\alpha)$. Assume that $\tau_{2} \geq 2\norm{\nabla \tilde{\ell}(\bT)}$ 
and $\tilde{\ell}(\bS) + \tau_{2}\norm{\bS}_{\star} \leq \tilde{\ell}(\bT) - \tau_{1} \norm{\bA( \bS_{[d]} - \bT_{[d]} )  }_{F}^{2}  + \tau_{2}\norm{\bT}_{\star}$, where  $\bS_{[d]}$ consists of the first $d$ columns of $\bS$. Then we have
\begin{enumerate}
\item $\Norm{ \mathcal{P}_{\bT}^{\perp}(\bS - \bT) }_{\star}  \leq
 3 \Norm{  \mathcal{P}_{\bT}(\bS - \bT)   }_{F} - 
 2\tau_{1}\tau_{2}^{-1} \sigma_{\min}^{2}(\bA)\Norm{ \bS_{[d]} - \bT_{[d]}}_{F}^{2}     $,
 \item $\Norm{\bS - \bT}_{\star} \leq 4\sqrt{2\mbox{rank}(\bT)} \Norm{\bS - \bT}_{F} - 
 2\tau_{1}\tau_{2}^{-1} \sigma_{\min}^{2}(\bA)\Norm{ \bS_{[d]} - \bT_{[d]}}_{F}^{2}     $.
\end{enumerate}
\end{lemma}

\begin{proof}[Proof]
Rearranging the terms, we have 
\begin{align}\label{L31}
  \tau_{2}\left(\Norm{\bS}_{\star} - \Norm{\bT}_{\star} \right) \leq \tilde{\ell}(\bT) - \tilde{\ell}(\bS) - \tau_{1} \Norm{\bA( \bS_{[d]} - \bT_{[d]} )  }_{F}^{2}.
\end{align}
By convexity of $\tilde{\ell}(\cdot)$ and \eqref{trace inequality}, we have
\begin{align}\label{L32}
\tilde{\ell}(\bT) - \tilde{\ell}(\bS)  \leq \langle  \nabla\tilde{\ell}(\bT), \bT - \bS \rangle 
\leq \Norm{\nabla\tilde{\ell}(\bT)} \Norm{\bT - \bS}_{\star}  \leq   \frac{\tau_{2}}{2}\Norm{\bT-\bS}_{\star}.
\end{align}
Plug \eqref{L32} in \eqref{L31} with \eqref{nuclear difference}, and we obtain
\begin{align}
\tau_{2}\Norm{ \mathcal{P}_{\bT}^{\perp}(\bS - \bT)}_{\star} - \tau_{2}\Norm{ \mathcal{P}_{\bT}(\bS - \bT)}_{\star}
\leq  \frac{\tau_{2}}{2}\Norm{\bT-\bS}_{\star} - \tau_{1} \Norm{\bA( \bS_{[d]} - \bT_{[d]} )  }_{F}^{2}.\notag
\end{align}
Rearranging terms, we have
\begin{align}
 &\Norm{ \mathcal{P}_{\bT}^{\perp}(\bS - \bT)}_{\star} \notag\\
 \leq&  3  \Norm{ \mathcal{P}_{\bT}(\bS - \bT)}_{\star} - 
 \frac{2\tau_{1}}{\tau_{2}} \Norm{\bA( \bS_{[d]} - \bT_{[d]} )  }_{F}^{2} \notag \\
 \leq&  3  \Norm{ \mathcal{P}_{\bT}(\bS - \bT)}_{\star} - 
 \frac{2\tau_{1}}{\tau_{2}} \sigma_{\min}^{2}(\bA) \Norm{( \bS_{[d]} - \bT_{[d]} )  }_{F}^{2},\notag
\end{align}
where the first part is proved.
For the second part, we have
\begin{align}
 &\Norm{\bS - \bT}_{\star} \notag\\
 =& \Norm{  \mathcal{P}_{\bT}^{\perp} (\bS - \bT) + \mathcal{P}_{\bT} (\bS - \bT)   }_{\star}\notag\\
 \leq& \Norm{  \mathcal{P}_{\bT}^{\perp} (\bS - \bT)}_{\star} + \Norm{  \mathcal{P}_{\bT} (\bS - \bT)}_{\star} \notag \\
 \leq& 4\Norm{  \mathcal{P}_{\bT} (\bS - \bT)}_{\star}  - \frac{2\tau_{1}}{\tau_{2}} \sigma_{\min}^{2}(\bA) \Norm{( \bS_{[d]} - \bT_{[d]} )  }_{F}^{2} 
 \notag\\
 \leq& 4\sqrt{2\mbox{rank}(\bT)} \Norm{\bS - \bT}_{F} - 
 \frac{2\tau_{1}}{\tau_{2}} \sigma_{\min}^{2}(\bA)\Norm{ \bS_{[d]} - \bT_{[d]}}_{F}^{2}.\notag
\end{align}
\end{proof}

\begin{lemma}\label{lambda min}
Given $\bS$ and $\bT$ of arbitrary matrices with the same dimension,  
\begin{eqnarray}
  \sigma_{\min}(\bS) \Norm{\bT}_{F}  \leq \Norm{\bS\bT}_{F}.\notag
\end{eqnarray}
\end{lemma}

\begin{proof}
By singular value decomposition, let  $\bS = \bU_{\bS}\Lambda_{\bS}\bV_{\bS}\trans $ and $\bT = \bU_{\bT}\Lambda_{\bT}\bV_{\bT}\trans$. Then by the cyclic property of trace operator,
\begin{align}
 \Norm{\bS\bT}_{F}^{2} &= \Norm{  \bU_{\bS}\Sigma_{\bS}\bV_{\bS}\trans\bU_{\bT}\Sigma_{\bT}\bV_{\bT}\trans   }_{F}^{2}\notag \\
 &= \mbox{tr} \left(  \bV_{\bT} \Sigma_{\bT}\bU_{\bT}\trans  \bV_{\bS} \Sigma_{\bS}\bU_{\bS}\trans
            \bU_{\bS}\Sigma_{\bS}\bV_{\bS}\trans\bU_{\bT}\Sigma_{\bT}\bV_{\bT}\trans \right) \notag \\
 &= \mbox{tr} \left(    \bV_{\bS} \Sigma_{\bS}^{2}\bV_{\bS}\trans\bU_{\bT}\Sigma_{\bT}\bV_{\bT}\trans \bV_{\bT} \Sigma_{\bT}\bU_{\bT}\trans  \right)\notag \\
 &\geq \sigma_{\min}^{2}(\bS) \mbox{Tr}\left( \bU_{\bT}\Sigma_{\bT}\bV_{\bT}\trans \bV_{\bT} \Sigma_{\bT}\bU_{\bT}\trans\right )  \notag \\
&= \sigma_{\min}^{2}(\bS) \Norm{\bT}_{F}^{2}. \notag
\end{align}
\end{proof}

The following three lemmas can be  adapted from \cite{alaya2019collective}.
\begin{lemma}[Modified Lemma 5 of \cite{alaya2019collective}]\label{Rademacher Matrix}
Suppose that Assumption 1 holds. There exists a positive constant $c$, such that
\begin{align}
\mathbb{ E } \left(\Norm{\bSigma_{R}} \right) \leq \frac{ c\left\{ \sqrt{ \gamma} + \sqrt{ \log{ (n \vee D)     }  } \right\}   }{nD}.\notag
\end{align}
\end{lemma}


\begin{lemma}[Modified Lemma 6 of \cite{alaya2019collective}]\label{Gradient Bound}
Suppose that Assumption 1,
3 and 4 hold and  there exists a positive constant $c$, such that
\begin{align}
\Norm{\nabla \tilde{\ell}(\bM_{\star})} \leq \frac{ c(\sqrt{\tilde{U}_{\alpha}} \vee 1/\delta )\left[ \sqrt{\gamma}  + \left\{\log (n\vee D) \right\}^{3/2} \right]       }{nD}
\notag
\end{align}
holds with probability at least $1 - 4/(n+D)$.
\end{lemma}


\begin{lemma}[Modified Lemma 19 of \cite{alaya2019collective}]\label{Bound}
Suppose that Assumption $1 \sim 3$ hold. Let $\beta = 946(p_{\min}nD)^{-1} \alpha^{2} \log(n+D)$. Let
\begin{align}
  \mathcal{H}(\bM_{\star}, \alpha, \beta, \mu, \theta) =&
  \left\{ \bH \in \mathcal{G}_{\infty}(\alpha): \Norm{\bH - \bM_{\star}}_{\star}  \leq  \sqrt{\mu} \Norm{\bH - \bM_{\star}}_{F}-
   \theta  \Norm{\bH_{[d]} - \bM_{\star[d]}}_{F}^{2},\right.\notag\\
   &\left.\quad\quad \frac{1}{nD}   \Norm{\wh{\bM} - \bM_{\star}}_{\Pi,F}^{2} > \beta      \right\}.\notag
\end{align}
Then, for any $\bH \in \mathcal{H}(\bM_{\star}, \alpha, \beta, \mu, \theta)$,
\begin{align}
\Abs{ \Delta_{\tilde{\ell}}(\bH, \bM_{\star})   - \frac{1}{nD} \Norm{ \bH - \bM_{\star} }_{\Pi,F}^{2}   } 
\leq \frac{ 1  }{2nD}  \Norm{\bH - \bM_{\star}}_{\Pi,F}^{2} +  \frac{1392nD\alpha^{2}\mu}{p_{\min}} \left\{\mathbb{E}(\Norm{\bSigma_{R}})  \right\}^{2}+ 
\frac{5567\alpha^{2}}{nDp_{\min}} ,\notag
\end{align}
with probability at least $1- 4(n+D)^{-1}$.
\end{lemma}

\section{Computational Time Comparison}\label{SectionE}

In terms of accuracy, our algorithm outperforms others. For the computational time, 
our algorithm enjoys a sublinear rate, which is the same as the \textbf{CMS\_SI} method. Besides, 
\textbf{TS} and \textbf{MC\_0} methods are just small deviations from \textbf{TMCC} method, and they enjoy the same sublinear convergence rate.
We present the computational behavior among different methods under the missing rate $\nu = 80\%$,  rank $r$ = 15 with $50$ trials. Further, the stopping criterion for the objective is $\kappa=1e-7$. For \textbf{TS}, the first stage (Feature matrix recovery) stopping criterion is $\kappa_{0} = 1e-12$. The results are presented in Table~\ref{tab:computation}, where the running time for $\textbf{TS}$ consists of the time for the first stage plus that for the second stage.   The results show the comparable elapsed time for the four methods.

  \begin{table}[H]
        \centering
          \renewcommand{\arraystretch}{1.3}
               \caption{Computational Time for Different Scenarios.}
           \label{tab:computation}
       \begin{adjustbox}{max width=1.0\linewidth}
       \begin{tabular}{ccrcrr}
       \toprule
        \bfseries \bfseries Transformation &
    \bfseries    Measure & \bfseries CMC\_SI &
        \bfseries TS & \bfseries MC\_0 & \bfseries TMCC\\
        \midrule
         \multirow{2}{*}{Linear} & Average Time(s) & 569.82 & 1865.40 (185.19+1680.21) & 676.20 &  926.53 \\
          & Standard Error & 78.19& 801.61 (107.40+753.79) & 50.86 & 125.39 \\
          \multirow{2}{*}{Nonlinear} & Average Time(s) & 271.78 & 1072.69 (201.79+870.90) & 1141.88&  1108.50 \\
          & Standard Error & 28.13 &323.55 (110.12+255.70) & 476.88& 247.77 \\
       \bottomrule
       \end{tabular}
       \end{adjustbox}
       \end{table}

\end{document}